\def\year{2018}\relax
\documentclass[letterpaper]{article} 
\usepackage{aaai18}  
\usepackage{times}  
\usepackage{helvet}  
\usepackage{courier}  
\usepackage{url}  
\usepackage{graphicx}  
\usepackage{booktabs}
\usepackage{subcaption}
\usepackage{adjustbox}
\usepackage{multirow}
\usepackage{lineno}
\usepackage{amsmath}
\usepackage{algorithmicx}
\usepackage{algorithm}
\usepackage{nicefrac} 
\usepackage{color}
\usepackage{dsfont}
\usepackage{amsfonts}
\usepackage{amssymb}
\usepackage{amsthm}
\usepackage{epsfig}
\usepackage{enumerate}
\usepackage{placeins}

\newtheorem{theorem}{Theorem}

\newtheorem{definition}{Definition}
\newtheorem{assumption}{Assumption}
\newtheorem{remark}{Remark}

\newcommand{\citet}[1]
{\citeauthor{#1}~\shortcite{#1}}
\newcommand{\citep}{\cite}

\begin{document}
%
\title{Nonparametric Stochastic Contextual Bandits}

\author{Melody Y. Guan\thanks{Equal Contribution.}\\
Stanford University\\
450 Serra Mall\\
Stanford, California 94305\\
mguan@stanford.edu
\And
Heinrich Jiang\footnotemark[1]\\
Google\\
1600 Amphitheatre Pwky\\
Mountain View, California 94043\\
heinrich.jiang@gmail.com
}
\maketitle

\begin{abstract}
We analyze the $K$-armed bandit problem where the reward for each arm is a noisy realization based on an observed context under mild nonparametric assumptions.
We attain tight results for top-arm identification and a sublinear regret of $\widetilde{O}\Big(T^{\frac{1+D}{2+D}}\Big)$, where $D$ is the context dimension, for a modified UCB algorithm that is simple to implement ($k$NN-UCB). 
We then give global intrinsic dimension dependent and ambient dimension independent regret bounds. We also discuss recovering topological structures within the context space based on expected bandit performance and provide an extension to infinite-armed contextual bandits. Finally, we experimentally show the improvement of our algorithm over existing multi-armed bandit approaches for both simulated tasks and MNIST image classification.
\end{abstract}

\section{Introduction}
Multi-armed bandits (MABs) are an important sequential optimization problem introduced by \citet{robbins1985some}. These models have extensively been used in a wide variety of fields related to statistics and machine learning. 

The classical MAB
consists of $K$ arms where at each point in time the learner can sample (or pull) one of them and observe a reward. Then various objectives can be established, such as finding the best arm (Top-Arm Identification) or minimizing some regret over time.

For contextual bandits (also referred to as bandits with side information or covariates), the learner has access to a context on which the payoffs depend. Then, based on the observations, we aim to determine the best policy (or context-to-arm mapping) and to optimize some notion of regret. 

Most approaches to stochastic contextual bandits make strong assumptions on the payoffs. A popular approach models the mean reward for each arm as being linear in the context space \cite{chu2011contextual,li2010contextual}. However, this is rarely the case in real data. In this paper, we take a more general approach and allow the reward functions to be non-linear and of arbitrary shape. 

Using recent developments in nonparametric statistics \cite{jiang2017rates}, we show that with simple and easily implementable techniques, we can construct bandit algorithms which can learn over the entire context space with strong guarantees, despite the difficulty that arises with allowing a wide variety of reward functions. While this is not the first work which blends nonparametric statistics with bandits, we are the first to show simple and practical methods while still maintaining strong theoretical guarantees.

We reanalyze the uniform and upper confidence bound sampling strategies and demonstrate what nonparametric approaches can offer to contextual bandit learning. No other technique can adapt to the inherently difficult and complex real world reward functions while allowing such a strong theoretical understanding of the underlying algorithms.  

While nonparametric models are powerful in their ability to learn arbitrary functions free of distributional assumptions, a major weakness is the curse of dimensionality. In order to have any theoretical guarantees, they require an exponential-in-dimension number of samples. However, when the data lies on an unknown low-dimensional structure such as a manifold, we show that our algorithms can converge as if the data was on a lower dimension and not in the potentially much large ambient dimension. Another striking fact is that no preprocessing of the data is required. This is of practical importance because modern data has increasingly more features but the underlying degrees of freedom often remain small.

We then discuss recovering geometric structures in the context space based on bandit performance. Specifically, we recover the connected components of the context space in which a particular bandit is the top-arm. Although learning a context-to-arm mapping gives us the estimated top-arm at each point in the context space, this alone does not tell the space's topological structure, such as the number and shapes of connected components. We recover these structures with uniform consistency guarantees with mild assumptions, where the shapes and relative positions of the components can be arbitrary and the number of such components is recovered automatically.

We then provide an extension to infinite-armed bandits and conclude with empirical results from simulations and image classification on the MNIST dataset. 

\section{Setup}
Suppose there are $K$ bandit arms indexed in $[K]$. 
At each time-step $t$, the learner observes a context $x_t \in \mathbb{R}^D$ where $x_t$ is drawn i.i.d. from a context density $p_X$ with compact support $\mathcal{X}$ bounded below away from zero (e.g. $\inf_{x \in \mathcal{X}} p_X(x) \ge p_0$ for some $p_0 \ge 0$). 
Then the learner chooses an arm $I_t \in [K]$ 
and observes reward $$r_t = f_{I_t}(x_t) + \xi_t$$ where
$\xi_t$ is drawn according to white noise random variable $\xi$ and $f_i : \mathcal{X} \rightarrow \mathbb{R}$ is the $i$-th arm's mean reward.
We make the following assumptions.

\begin{assumption} (Lipschitz Mean Reward)
There exists $L$ such that $|f_i(x) - f_i(x')| \le L|x-x'|$ for all $x,x' \in \mathcal{X}$ and $i \in [K]$.
\end{assumption}

\begin{assumption} (Sub-Gaussian White noise)
$\xi$ satisfies
$E[\xi] = 0$ and is sub-Gaussian with parameter $\sigma^2$ (i.e. $E[\exp(\lambda\xi)] \le \exp(\sigma^2\lambda^2/2)$ for all $\lambda \in \mathbb{R}$).
\end{assumption}

We require the finite-sample strong uniform consistency 
result (Theorem~\ref{theo:knn}) for $k$-NN regression defined as fellows:
\begin{definition} [$k$-NN]
Let the $k$-NN radius of $x \in \mathcal{X}$ be $r_k(x) := \inf \{ r : |B(x, r) \cap X| \ge k \}$ where $B(x, r) := \{x \in \mathcal{X} : |x - x'| \le r \}$ and the $k$-NN set of $x \in \mathcal{X}$ be $N_k(x) := X \cap B(x, r_k(x))$. 
Then for $x \in \mathcal{X}$,
\begin{align*} 
\widehat{f}_{\text{k-NN}}(x) := \frac{1}{|N_k(x)|} \sum_{i=1}^n y_i \cdot 1\left[ x_i \in N_k(x) \right].
\end{align*}
\end{definition}
\begin{theorem}(Rate for $k$-NN \cite{jiang2017rates}) \label{theo:knn}
Let $\delta > 0$.
There exists $N_0$ and universal constant $C$ such that if $n \ge N_0$ and $k = \lfloor n^{2/(2+D)}\rfloor$, then with probability at least $1 - \delta$,
\begin{align*}
\sup_{x \in \mathcal{X}} |f(x) - \widehat{f}_{\text{k-NN}}(x)| \le C \sqrt{\log n \log(1/\delta)}\cdot n^{-1/(2+D)}.
\end{align*}
\end{theorem}

It will be implicitly understood from here on that $\widehat{f}_i$ denotes the $k$-NN regression estimate of $f_i$ under the settings of Theorem~\ref{theo:knn}.

\section{Top-Arm Identification}
\begin{algorithm}[tbh]
  \caption{Uniform Sampling
    \label{uniform}}
  \begin{algorithmic}[1]
\State Parameters: $T$, total number of time steps. 
\State For each arm $i$ of the $K$ arms:
\State  \indent For each time step $t\in[\frac{(i-1) T}{K}+1,\frac{iT}{K}]$:
\State  \indent \indent {Pull arm $I_t := i$.}
\State Define $\widehat{f}_{i} : \mathcal{X} \rightarrow \mathbb{R}$ to be the $k$-NN regression estimator from the sampled context and reward observations for each $i \in [K]$.
  \end{algorithmic}
\end{algorithm}

\begin{definition} ($\epsilon$-optimal arm)
Arm $i$ is be $\epsilon$-optimal at context $x \in \mathcal{X}$ if 
$\max_{j \in [K]} f_j(x) - f_i(x) \le \epsilon$.
\end{definition}

Following we show a uniform (over context) result about $\epsilon$-optimal arm recovery:
\begin{theorem}($\epsilon$-optimal arm recovery)\label{theo:uniform}
Let $\delta > 0$.
For Algorithm~\ref{uniform},
with probability at least $1 - \delta/K$, if 
{\small
\begin{align*}
    T &\ge  K\max \bigg\{N_0, \\ &\log\left(\frac{C\sqrt{\log(1/\delta)}}{\epsilon}\right) \cdot \frac{(2+D) (2C)^{2+D}\log(1/\delta)^{1+D/2}}{\epsilon^{2+D}}\bigg\},
\end{align*}
}%
 then $\hat{\pi}(x) := \text{argmax}_{i\in [K]} \hat{f}_{i}(x)$ is $\epsilon$-optimal at context $x$ uniformly for all $x \in \mathcal{X}$.
\end{theorem}

\begin{remark}
This result shows that with $\widetilde{O}(\epsilon^{-(2+D)})$ samples, we can determine an $\epsilon$-approximate best arm. 
Known lower bounds in nonparametric regression stipulate that we need $\Omega(\epsilon^{-(2+D)})$ to identify differences between functions of size $\epsilon$ so our result matches lower bounds up to logarithmic factors.
\end{remark}

\begin{proof}
By Theorem~\ref{theo:knn}, it follows that based on the choice of $T$, each arm has at least enough time such that $\sup_{x \in \mathcal{X}} |\widehat{f}_i(x) - f_i(x)| \le \epsilon/2$. Thus,
we have $\forall x \in \mathcal{X}$, defining $\pi(x) =  \max_{j \in [K]} f_j(x)$,
\begin{align*}
   f_{\pi(x)}(x) - f_{\hat{\pi}(x)}(x)
\le \widehat{f}_{\pi(x)}(x) -  \widehat{f}_{\hat{\pi}}(x) + \epsilon \le \epsilon,
\end{align*}
as desired.
\end{proof}

\section{Regret Analysis For UCB Strategy}

Define $T_i(t)$ to be the number of times arm $i$ was pulled by time $t$.

\begin{algorithm}[tbh]
  \caption{Upper Confidence Bound (UCB)
    \label{ucb}}
  \begin{algorithmic}[1]
  \State Parameters: $M_0$, $M_1$, $\delta$, $T$.
  \State Define $\sigma(n) = M_1 \sqrt{\log n(\log(nK/\delta))} \cdot n^{-1/(2+D)}$.
\State Pull each of the $K$ arms $M_0$ times.
\State For each round $t = K M_0, KM_0 +1, \dots, T$:
\State  \indent  Pull $I_t := \text{argmax}_{i \in [K]} \widehat{f}_i(t) + \sigma (T_i(t - 1))$.
  \end{algorithmic}
\end{algorithm}

We use the following notion of regret.
\begin{align*}
    R_T = \sum_{t=1}^T \big[\max_{i} f_i(x_t) - f_{I_t}(x_t)]\big]. 
\end{align*}

\begin{remark}
    Note that this notion of regret is different from
    those studied in classical MABs as well as other works in nonparametric contextual bandits. Usually the expected form $E[R_T]$ is bounded. Here, our regret analysis is not under this expectation and hence is a stronger notion of regret. 
\end{remark}

\begin{theorem}\label{theo:ucb}
Let $\delta > 0$.
Suppose that $M_0 \ge N_0$ and $M_1 > C$ in Algorithm~\ref{ucb}.
Then we have that with probability at least $1 - \delta$,
\begin{align*}
R_T \le &M_1 2\frac{1+D}{2+D} K\sqrt{\log T(\log(TK/\delta)} \cdot T^{\frac{1+D}{2+D}} \\
&+ K M_0 \max_i ||f_i||_\infty.
\end{align*}
\end{theorem}
\begin{remark}
This shows a sub-linear regret of $\widetilde{O}(T^{\frac{1+D}{2+D}})$.
\end{remark}
\begin{proof}
Denote $\widehat{f}_{i, T_i(t)}$ to be the $k$-NN regression estimate of $f_i$ at time $t$.
Letting $C_0 = K M_0 \max_i ||f_i||_\infty$, we have by Theorem~\ref{theo:knn}
\begin{align*}
    R_T &\le \sum_{i=1}^T \sigma(T_{\hat{\pi}(x_t)}(t-1)) +C_0 \le K \sum_{i=1}^T \sigma(i)  +C_0 \\
    &= M_1 K\sqrt{\log T(\log(TK/\delta)} \sum_{t=1}^T t^{-1/(2+D)} + C_0 \\
    &\le  M_1 K\sqrt{\log T(\log(TK/\delta)} \int_{t=0}^T (1+t)^{-1/(2+D)} dt\\
    &\hspace{1em} + C_0 \\
    &\le  M_1 2\frac{1+D}{2+D} K\sqrt{\log T(\log(TK/\delta)} \cdot T^{\frac{1+D}{2+D}}  + C_0.
\end{align*}
The first inequality holds because the confidence bound of a sub-optimal arm must be higher than that of the optimal at $x_t$ in order for that arm to be chosen and the regret at that time-step is bounded by the confidence bound. The second inequality holds because of the following simple combinatorial argument. Each time a suboptimal arm is chosen, its count increments, or otherwise there is no regret incurred.
\end{proof}

\section{Contextual Bandits on Manifolds}

\begin{assumption}\label{manifold}(Manifold Assumption)
$p_X$ and the family of $f_i$ are supported on $M$, where:
\begin{itemize}
\item $M$ is a $d$-dimensional smooth compact Riemannian manifold without boundary embedded in compact subset $\mathcal{X} \subseteq \mathbb{R}^D$.
\item The volume of $M$ is bounded above by a constant.
\item $M$ has condition number $1/\tau$, which controls the curvature and prevents self-intersection.
\end{itemize}
Let $p_X$ be the density of $\mathcal{P}$ with respect to the uniform measure on $M$.
\end{assumption}

\begin{theorem}(Manifold Rate for $k$-NN \cite{jiang2017rates}) \label{theo:knnmanifold}
Let $\delta > 0$.
There exists $N_0$ and universal constant $C$ such that if $n \ge N_0$ and $k = \lfloor n^{2/(2+d)}\rfloor$, then with probability at least $1 - \delta$,
\begin{align*}
\sup_{x \in \mathcal{X}} |f(x) - f_k(x)| \le C \sqrt{\log n \log(1/\delta)}\cdot n^{-1/(2+d)}.
\end{align*}
\end{theorem}
Then, simply by using Theorem~\ref{theo:knnmanifold} instead of Theorem~\ref{theo:knn}, we automatically enjoy faster rates for Theorems~\ref{theo:uniform} and~\ref{theo:ucb}.

\begin{theorem}($\epsilon$-optimal arm recovery on manifolds)
Let $\delta > 0$.
For Algorithm~\ref{uniform},
with probability at least $1 - \delta/K$, if
{\small
\begin{align*}
    T &\ge  K\max \bigg\{N_0, \\ &\log\left(\frac{C\sqrt{\log(1/\delta)}}{\epsilon}\right) \cdot \frac{(2+D) (2C)^{2+d}\log(1/\delta)^{1+D/2}}{\epsilon^{2+d}}\bigg\},
\end{align*}
}%
 then $\hat{\pi}(x) := \text{argmax}_{i\in [K]} \widehat{f}_{i}(x)$ is $\epsilon$-optimal at context $x$ uniformly for all $x \in \mathcal{X}$.
\end{theorem}
\begin{remark}
Now the sample complexity is $\widetilde{O}(\epsilon^{2+d})$ instead of $\widetilde{O}(\epsilon^{2+D})$.
\end{remark}
\begin{theorem}(UCB Regret Analysis on Manifolds)\label{theo:ucbmanifold}
Let $\delta > 0$.
Suppose that $M_0 \ge N_0$ and $M_1 > C$ in Algorithm~\ref{ucb}.
Then we have that with probability at least $1 - \delta$,
\begin{align*}
R_T \le &M_1 2\frac{1+d}{2+d} K\sqrt{\log T(\log(TK/\delta)} \cdot T^{\frac{1+d}{2+d}}\\
&+ K M_0 \max_i ||f_i||_\infty.
\end{align*}
\end{theorem}

\section{Topological Analysis}
In this section, we discuss how topological features about the bandit arms can be recovered. This is similar to recovering the Hartigan notion of clusters as level-sets of the density functions from a finite sample \cite{chaudhuri2010rates,jiang2017density}, but here, we find similar structures in the reward functions based on noisy observations of them.
We give procedures which can estimate with consistency guarantees the following structure: maximal connected regions in $\mathcal{X}$ where a particular arm is the top-arm.

From the uniform sampling strategy earlier, we obtained estimated policy $\hat{\pi}$ which is $\delta$-optimal uniformly in $\mathcal{X}$ with high probability. Although this is already powerful in giving us the mapping between context space and the corresponding top-arm, it does not immediately tell us the topological features of this mapping. In this subsection, we discuss how to recover the connected components of $\{ x \in \mathcal{X} : r_i(x)  = \max_{j \in [K]} r_j(x) \}$, the region where arm $i$ is the top-arm.

We give the following simple procedure. 

\begin{algorithm}[tbh] 
  \caption{Recovering Regions where $i$-th arm is top arm.}
  \label{alg:clustering1}
  \begin{algorithmic}[1]
\State Given: Bandit arm $i$ and $R > 0$.
\State Pull each of the $K$ arms $T/K$ times.
\State Let $G$ be the graph with vertices $\{ x_t : t \in [T], \widehat{f}_i(x_t) = \max_{j \in [K]} \widehat{f}_j(x_t) \}$ and edges between vertices whose euclidean distance is at most $R$.
\State {\bf return } The connected components of $G$.
  \end{algorithmic}
\end{algorithm}

We now give a consistency result for Algorithm~\ref{alg:clustering1}. 

First, we require the following regularity assumption, which ensures that there are no full-dimensional regions where the top-arm is not unique. This ensures that it is possible to unambiguously recover the regions where a particular arm is top.
\begin{assumption}
The region in $\mathcal{X}$ where the top-arm is not unique has measure $0$, and for each arm $i$, the region $\mathcal{X}_i$ where it is unique can be partitioned into full-dimensional connected components.
\end{assumption}

Our rates will be in terms of the Hausdorff distance.
\begin{definition}
\begin{align*}
d_H(A, B) = \inf\{\epsilon \ge 0: A \subseteq B \oplus \epsilon, B \subseteq A \oplus \epsilon\},
\end{align*}

where $A \oplus r := \{x \in \mathcal{X} : \inf_{a \in A} d(x, a) \le r\}$.
\end{definition}

\begin{theorem}\label{theo:clustering1}
Suppose that $\mathcal{X}_i := \{ x \in\mathcal{X} : f_i(x)  = \max_{j \in [K]} f_j(x)\}$. Let $C_1,...,C_l$ be the maximal connected components of $\mathcal{X}_i$. 
Define the following minimum distance between two connected components.
\begin{align*}
    R_0 := \min_{p \neq q} \inf_{x \in C_p, y \in C_q} d(x, y).
\end{align*}
Also define the following minimum separation in the reward functions 
\begin{align*}
    D_0 := \inf_{x \not\in \mathcal{X}_i \oplus R_0/4} \max_{j \in [K]} f_j(x) - f_i(x).
\end{align*}
Then the following holds simultaneously for all $C_1,...,C_l$. 
Let Algorithm~\ref{alg:clustering1} with setting $0 < R < R_0/4$ return $\widehat{C}_1,...,\widehat{C}_{\hat{l}}$.
Then for $n$ sufficiently large, $\hat{l} = l$ and there exists permutation $\gamma$ of $[l]$ such that
\begin{align*}
    d_H(C_j, C_{\gamma(j)}) \le \xi(n)
\end{align*}
for some $\xi$ that satisfies $\xi(n) \rightarrow 0$ as $n\rightarrow \infty$.
\end{theorem}

\begin{proof}
We first show that no two connected components can appear in the same returned component in Algorithm~\ref{alg:clustering1}. We choose $n$ sufficiently large such that in light of Theorem~\ref{theo:knn}, we have $$\sup_{x \in \mathcal{X}} \max_{j \in [K]} \widehat{f}_j (x) \le \frac{D_0}{3}.$$.
Then, uniformly for any $x \not\in \mathcal{X}_i \oplus R_0/4$, we have
\begin{align*}
\widehat{f}_i(x) &\le f_i(x) + \frac{D_0}{3} \le \max_{j \in [K]} f_j(x) - \frac{2D_0}{3} \\
&\le \max_{j \in [K]} \widehat{f}_j(x) - \frac{D_0}{3} < \max_{j \in [K]} \widehat{f}_j(x).
\end{align*}
Thus, $\mathcal{X}_i \oplus R_0/4$ is disjoint from the returned points. Since $R < R_0/4$, it follows that no two connected components points will appear in the same returned connected component from Algorithm~\ref{alg:clustering1}.

Next, we show that for each connected component $C_p$, there exists $\widehat{C}_q$ for some $q \in [\widehat{l}]$ such that
$d_H(\widehat{C}_q, C_p) \rightarrow 0$.
It suffices to show that for each $r > 0$, we have that for $n$ sufficiently large, 
$d_H(\widehat{C}_q, C_p) < r$.
There are thus two directions to show, that $\widehat{C}_p\subseteq C_q \oplus r$ 
and $C_q\subseteq \widehat{C}_p \oplus r$.
To show the first, define
\begin{align*}
    D_1 := \inf_{x \in (C_q\oplus r) \backslash (C_q\oplus (r/2))} \max_{j \in [K]} f_j(x) - f_i(x).
\end{align*}
Then choose $n$ sufficiently large such that in light of
Theorem~\ref{theo:knn}, we have $$\sup_{x \in \mathcal{X}} \max_{j \in [K]} |\widehat{f}_j (x) - f_j(x)| \le \frac{D_1}{3}.$$.
Then we have for all $x \in \widehat{C}_p$, if $x \neq  C_q \oplus r/2$, then 
\begin{align*}
\widehat{f}_i(x) \le f_i(x) + \frac{D_1}{3} 
\le  \max_{j \in [K]} f_j(x) - \frac{2D_1}{3}
< \max_{j \in [K]} \widehat{f}_j(x),
\end{align*}
thus, $x \in C_q \oplus r/2 \subseteq  C_q \oplus r $. The other direction follows from a similar argument. 

All that remains is to show that such points appear in in the same connected component in the graph computed by Algorithm~\ref{alg:clustering1}. This follows from uniform concentration bounds on balls (e.g. \citet{chaudhuri2010rates}).
\end{proof}

\section{Infinite-Armed Bandits}
In this section, we consider the setting where the action space $\mathcal{A}$ is no longer a finite set of bandits, but a compact subset of $\mathbb{R}^{D'}$ for some $D' > 0$.

We given analogous results for the uniform sampling top-arm identification and regret bounds for UCB-type strategy.

\begin{definition} (Mean Reward function)
$$f : \mathcal{X} \times \mathcal{A} \rightarrow \mathbb{R},$$
where $f(x, a)$ is the expected reward of action $a \in \mathcal{A}$ at context $x \in \mathcal{X}$.
\end{definition}
\begin{assumption} (Lipschitz Reward)
There exists $L > 0$ such that for all $x,x' \in \mathcal{X}$ and $a,a' \in \mathcal{A}$,
$|f(x, a) - f(x', a')| \le L|(x,a) - (x',a')|$,
where $(x, a)$ represents the $(D+D')$-dimensional concatenation of $x$ and $a$.
\end{assumption}

Then at each time $t$, the learner chooses arm $a_t \in \mathcal{A}$ and observes context $x_t\in\mathcal{X}$ and a stochastic reward
\begin{align*}
    R_T = f(x_t, a_t) + \xi_t,
\end{align*}
where $\xi_1,...$ are i.i.d. white noise with mean $0$ and variance $\sigma^2$.

\begin{algorithm}[tbh]
  \caption{Infinite-Armed Uniform Sampling 
    \label{uniforminfinite}}
  \begin{algorithmic}[1]
\State Parameters: $T$, total number of time steps. 
 \State For $t = 1,...,T$:
 \State \indent Pull $I_t$, sampled uniformly from $\mathcal{A}$.
 \State \indent Observe context $x_t$ and reward $R_t$.
\State Define $\hat{f}$ to be the $k$-NN regression estimate from samples $(a_1,R_1),...,(a_T,R_T)$ with setting $k = \lfloor n^{2/(2 +D + D')}  \rfloor$.
  \end{algorithmic}
\end{algorithm}

\begin{definition} ($\epsilon$-optimal arm)
Define arm $a \in \mathcal{A}$ to be $\epsilon$-optimal at context $x \in \mathcal{X}$ if 
$\sup_{a' \in \mathcal{A}} f(x, a') - f(x, a) \le \epsilon$.
\end{definition}

Following is a uniform (over context and action space) result about $\epsilon$-optimal arm recovery:
\begin{theorem}($\epsilon$-optimal arm recovery)\label{theo:uniforminfinite}
There exists constant $\tilde{C}_1, \tilde{C}_2$ such that the following holds.
Let $\delta > 0$.
For Algorithm~\ref{uniforminfinite},
with probability at least $1 - \delta$, we have that for 
\begin{align*}
    T \ge \tilde{C}_1 \log\left(\frac{\sqrt{\log(1/\delta)}}{\epsilon}\right)  \frac{\log(1/\delta)^{1 + (D+D')/2}}{\epsilon^{D + D' + 2} }+ \tilde{C}_2,
\end{align*}
 arm $\hat{\pi}(x) := \text{argmax}_{a \in \mathcal{A}} \hat{f}(x)$ is $\epsilon$-optimal at context $x$ uniformly for all $x \in \mathcal{X}$.
\end{theorem}

\begin{proof}
By Theorem~\ref{theo:knn}, it follows that based on the choice of $T$, there is enough time spent on pulling each arm such that $\sup_{a \in \mathcal{A}, x \in \mathcal{X}} |\hat{f}(x, a) - f(x, a)| \le \epsilon/2$. Thus,
we have $\forall x \in \mathcal{X}$, defining $\pi(x) =  \text{argmax}_{a \in \mathcal{A}} f(x, a)$,
\begin{align*}
  & f(x, \pi(x)) - f(x, \hat{\pi(x)})\\
&\le \frac{\epsilon}{2} + \hat{f}(x, \pi(x)) + \frac{\epsilon}{2} -  \hat{f}(x,\hat{\pi}(x))
\le \epsilon,
\end{align*}
as desired.
\end{proof}

\begin{algorithm}[tbh]
  \caption{Infinite-Armed Upper Confidence Bound (UCB)
    \label{ucbinfinite}}
  \begin{algorithmic}[1]
  \State Parameters: $M$, $M_1$, $T$
      \State Define $\sigma(n) = M_1 n^{-1/(2+D+D')}$.

  \State For $t = 1,...,M$:
  \State \indent Sample $a_t$ uniformly from $\mathcal{A}$.
  \State \indent Observe context $x_t$ and reward $R_t$.
  \State For {$t = M+1,...,T$}:
   \State \indent Choose $I_t := \text{argmax}_{a \in \mathcal{A}} \widehat{f}(x_t, a) + \sigma (t)$.
  \end{algorithmic}
\end{algorithm}

Finally, using the notion of regret
\begin{align*}
    R_T = \sum_{t=1}^T \big[\sup_{a \in \mathcal{A}} f(x_t, a) - f(x_t, a_t)\big], 
\end{align*}
we give the following result. The proof idea is similar to that of Theorem~\ref{theo:ucb} and is omitted here.
\begin{theorem}\label{theo:ucbinfinite}
There exists $\tilde{C}_1$ and $\tilde{C}_2$ such that the following holds.
Let $\delta > 0$.
Suppose that $M$ and $M_1$ are chosen sufficiently large in Algorithm~\ref{ucbinfinite} depending on $f$ and $\sigma$.
Then we have that with probability at least $1 - \delta$,
\begin{align*}
R_T \le \tilde{C}_1 \sqrt{\log T(\log(T/\delta)} \cdot T^{\frac{1+D+D'}{2+D+D'}} +\tilde{C}_2
\end{align*}
\end{theorem}
\begin{remark}
This shows a sub-linear regret of $\widetilde{O}(T^{\frac{1+D+D'}{2+D+D'}})$.
\end{remark}

\section{Related Works}
Canonical works for the standard bandit problem are \citet{lai1985asymptotically}; \citet{berry1985bandit};  \citet{gittins2011multi};  \citet{auer2002nonstochastic};  \citet{cesa2006prediction};  \citet{bubeck2012regret}.

Work in contextual bandits can be roughly classified into adversarial and stochastic approaches. Much of the former, initiated by \citet{auer2002nonstochastic}, assumes that there is an adversarial game between nature and the learner where, based on a context seen by both players, nature generates rewards for each arm at the same time the learner chooses an arm. Solutions typically involve game theoretical methods. In the stochastic approach, one assumes that the rewards for the arms are generated by a context-dependent distribution.

Approaches to modeling the arm rewards as a function of context are most commonly parametric. One of the most popular is that of linear payoffs, studied under a minimax framework \cite{goldenshluger2009woodroofe,goldenshluger2013linear}, with UCB-type algorithms \cite{chu2011contextual,li2010contextual,auer2002nonstochastic}, or with Thompson sampling \cite{agrawal2013thompson}.

However, it is often the case that the dependency between the payoffs and the contexts are complex and therefore difficult to capture with models such as linear payoffs, many of which requiring strong assumptions on the data. 
To alleviate this, we can go beyond parametric modeling and blend nonparametric statistics with contextual bandits. Despite the advantage of learning much more general context-payoff dependencies, this line of work has received far less attention. 

To the best of our knowledge, the first such work appeared in \citet{yang2002randomized}, who used histogram, $k$-NN, and kernel methods and showed asymptotic convergence rates. \citet{rigollet2010nonparametric};  \citet{perchet2013multi} then combined histogram-type binning techniques in nonparametric statistics to obtain strong regret guarantees for contextual bandits with optimality guarantees. 

\citet{lu2009showing} study an interesting setting where the reward depends on a Lipschitz measure which is jointly in the context and the action space. They provide upper and lower regret bounds based on a covering argument and give results in terms of the packing dimension. This is highly related to the infinite-armed bandit setting in the present work; we provide similar regret guarantees but with a simple and practical procedure.  

More recently, \citet{qian2016randomized}; \citet{qian2016kernel} use the strong uniform consistency properties of kernel smoothing regression to establish regret guarantees.

\citet{langford2008epoch};  \citet{dudik2011efficient} alternatively impose
neither linear nor smoothness assumptions on the mean reward function. The former propose a modification of an $\epsilon$-greedy policy and showed that expected regret converges to $0$ while the latter considers a finite class of policies.

In this paper, using recent finite-sample results about $k$-NN regression established in \citet{jiang2017rates}, we show that using the simple $k$-NN regression is an effective alternative approach. Moreover, unlike many other nonparametric techniques, $k$-NN {\it adapts} to a lower intrinsic dimension \cite{kpotufe2011k} and thus we show that our regret bounds can adapt to a lower intrinsic dimension automatically and perform as if we were operating in that lower dimensional space. 

\section{Experiments}

\subsection{Simulations}
We consider three two-arm bandit scenarios in the two-dimensional unit square, where $p_\mathcal{X}$ is uniform. 
We set arm $i\in\{1,2\}$ to be top in region $R_i$ respectively. Figure \ref{fig:Scenarios} illustrates the regions for the different scenarios. 
\begin{itemize}
    \item \textbf{Scenario 1 (Quintic Function):} We define two regions above and below a quintic function:
    \item \textbf{Scenario 2 (Smiley):} We use two circles and a semicircle to demarcate the regions in a "smiley face" pattern.
\end{itemize}
\begin{itemize}
    \item \textbf{Scenario 3 (Bullseye):} We define the regions using the alternating regions of four concentric circles centered in the support.
\end{itemize}
The true reward functions of the two arms are as follows.
\[
    f_i(x)= 
\begin{cases}
    1,& x\in R_i \\
    0.5,& x\in R_{j\neq i}
\end{cases}
\]
The learner observes the rewards with white noise random variable $\xi\sim\mathcal{N}(\mu=0,\sigma=0.5)$. 

We compare the performance of $k$-NN regression (nonparametric) and Ridge regression at top-arm identification and regret minimization in the three scenarios. Mirroring our theoretical discussion, we use uniform sampling for top-arm identification and UCB strategy for regret analysis. Note that Ridge regression with UCB is the LinUCB algorithm.

\begin{figure}[!htbp]
\centering
\subcaptionbox{Quintic} {
\includegraphics[width=0.27\linewidth]{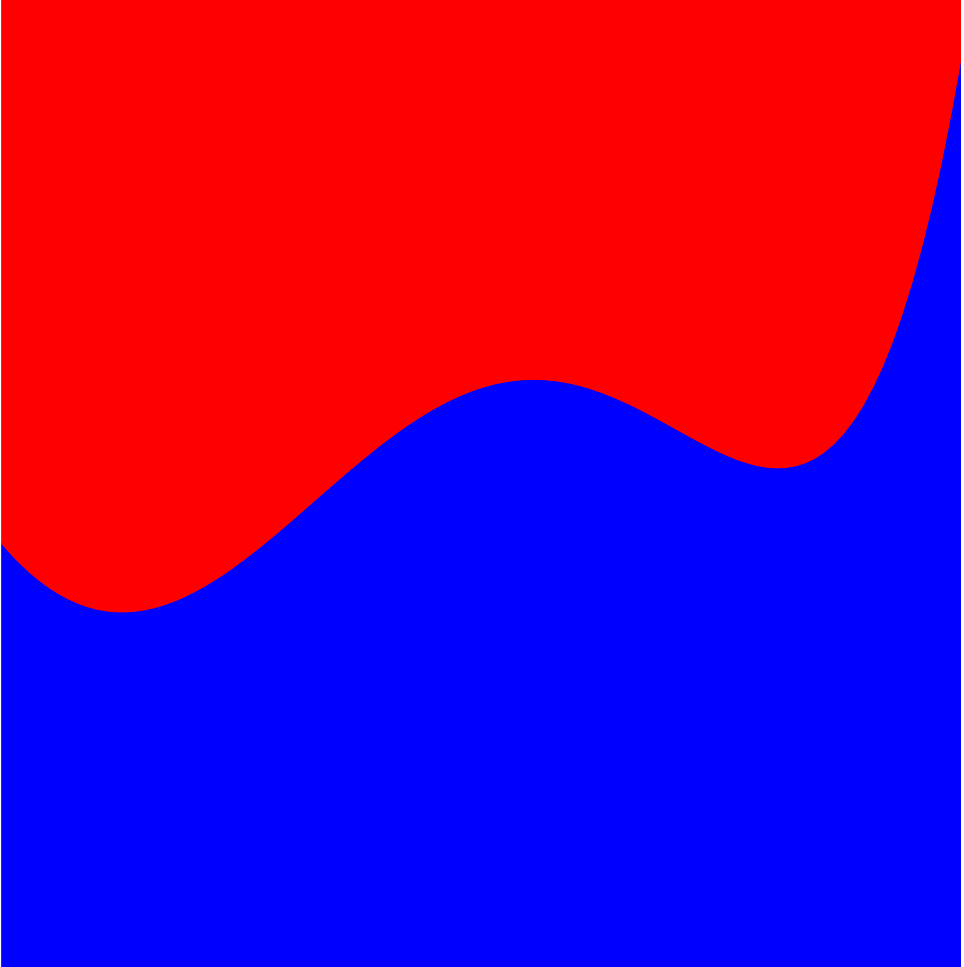}
}
\subcaptionbox{Smiley} {
\includegraphics[width=0.27\linewidth]{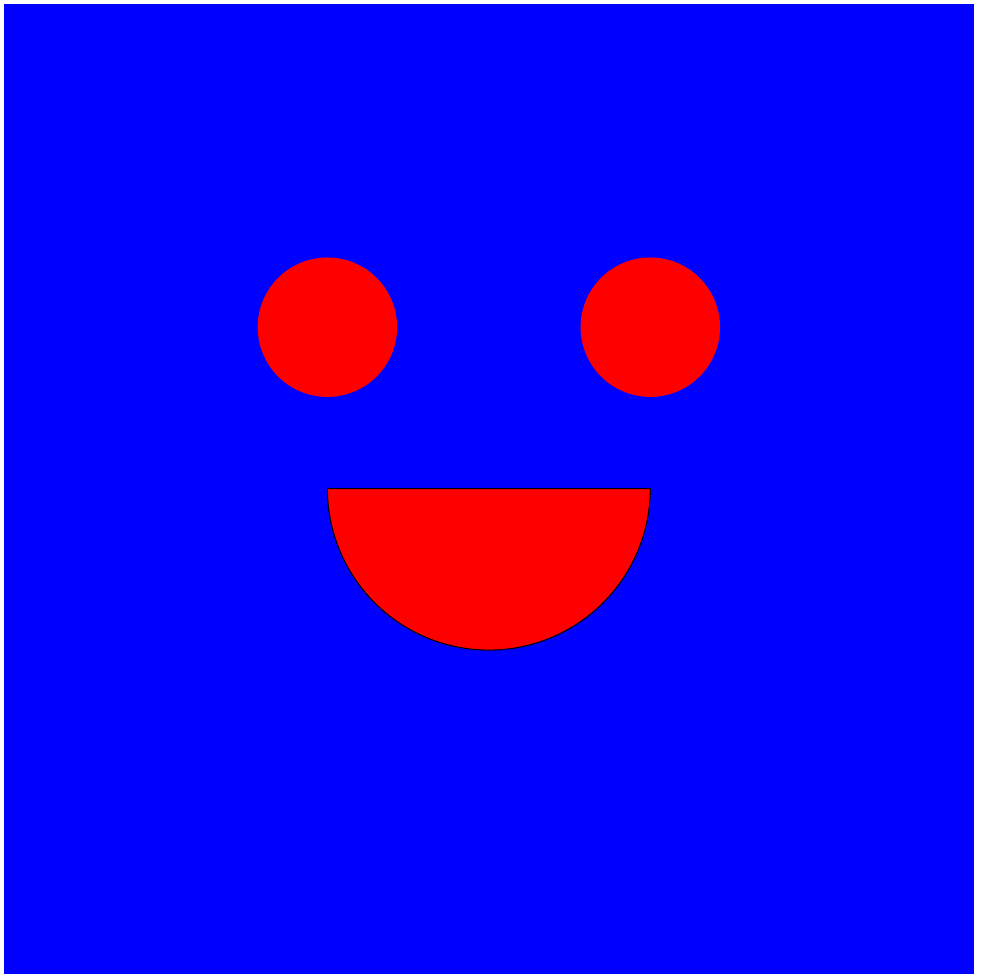}
}
\subcaptionbox{Bullseye} {
\includegraphics[width=0.27\linewidth]{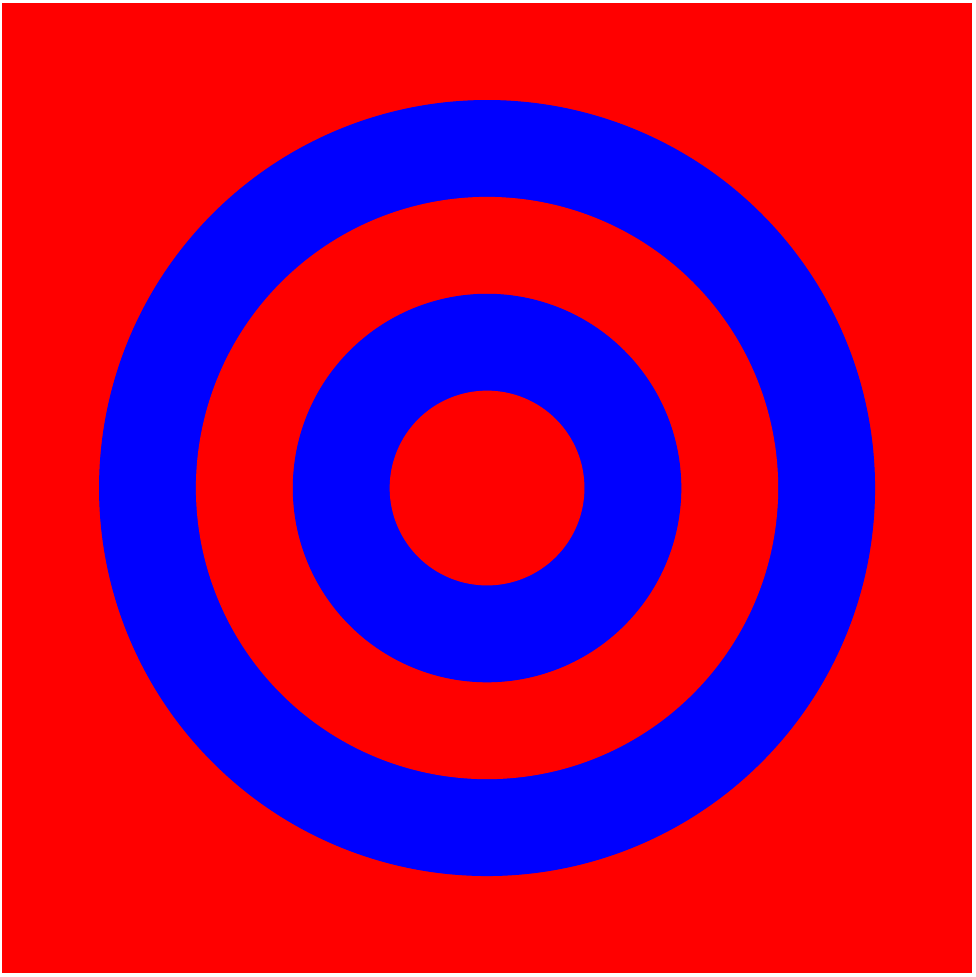}
}
\caption{Top-arm boundaries. Red and blue regions correspond to where top-arm is arm 1 and 2 respectively.}
\label{fig:Scenarios}
\end{figure}

\begin{figure}[!htbp]
\centering
\subcaptionbox{Quintic} {
\includegraphics[width=0.26\linewidth]{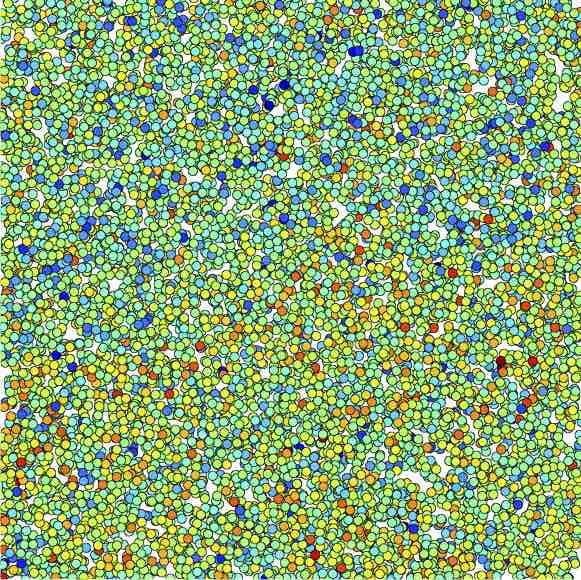}
}
\subcaptionbox{Smiley} {
\includegraphics[width=0.26\linewidth]{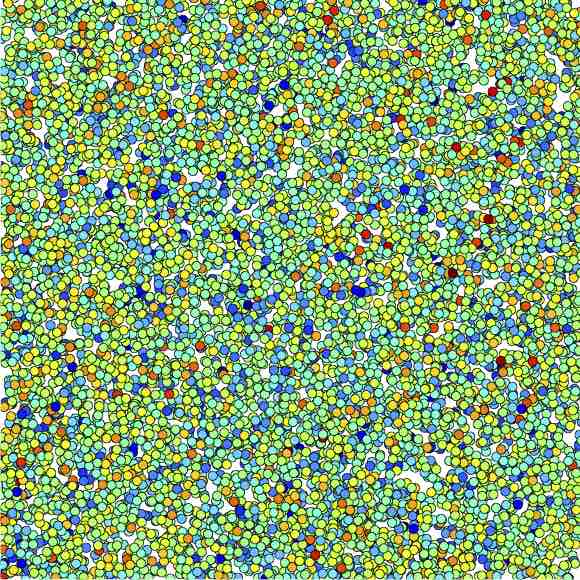}
}
\subcaptionbox{Bullseye} {
\includegraphics[width=0.26\linewidth]{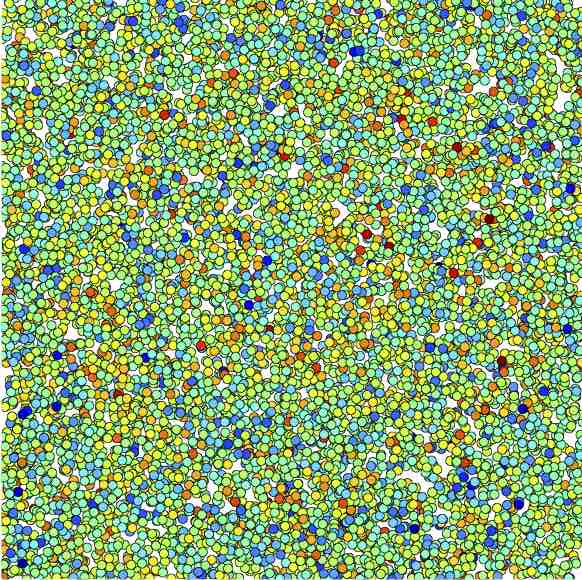}
}
\includegraphics[height=0.26\linewidth]{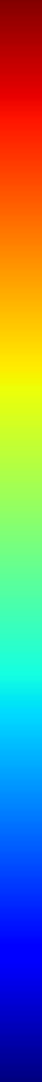}
\caption{\label{fig:RewardDensity}
Observed reward density plots from 10k uniform samples illustrating pseudo-randomness of training data. In the colormap (right) warmer colors correspond to higher values, normalized on the range of the observed rewards.}
\end{figure}

\begin{figure}[!htbp]
\centering
\subcaptionbox{Quintic Ridge} {
\includegraphics[width=0.27\linewidth]{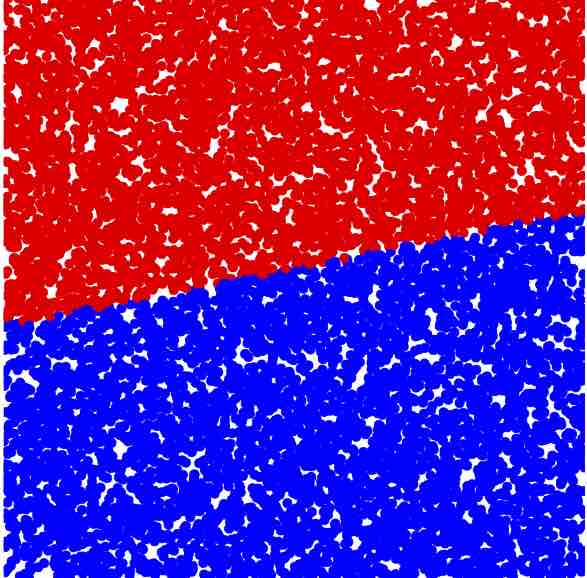}
}
\subcaptionbox{Smiley Ridge} {
\includegraphics[width=0.27\linewidth]{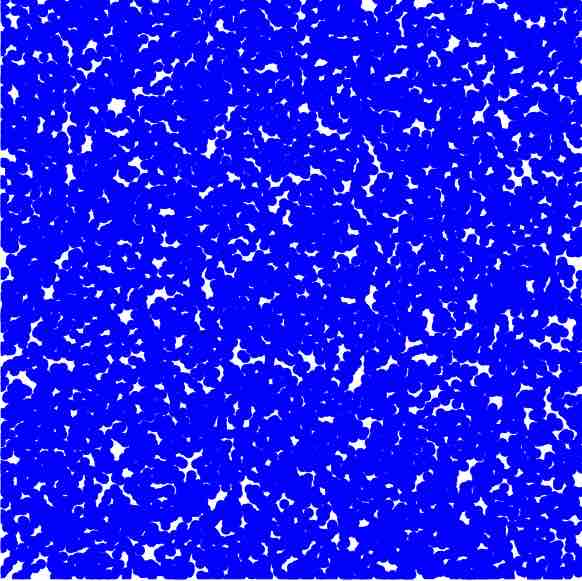}
}
\subcaptionbox{Bullseye Ridge} {
\includegraphics[width=0.27\linewidth]{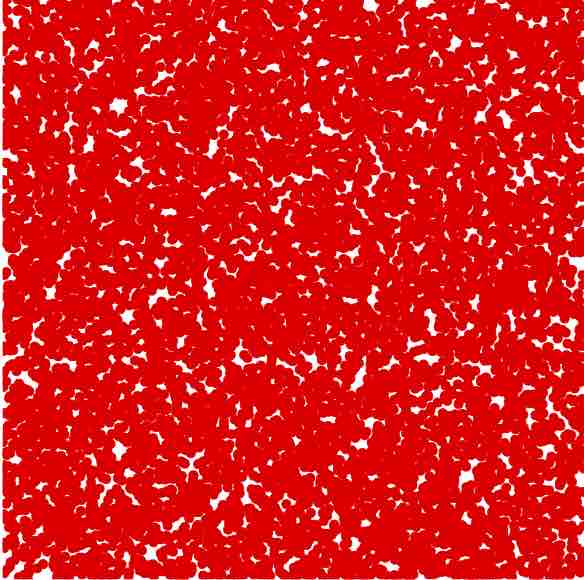}
}\\
\subcaptionbox{Quintic $k$-NN} {
\includegraphics[width=0.27\linewidth]{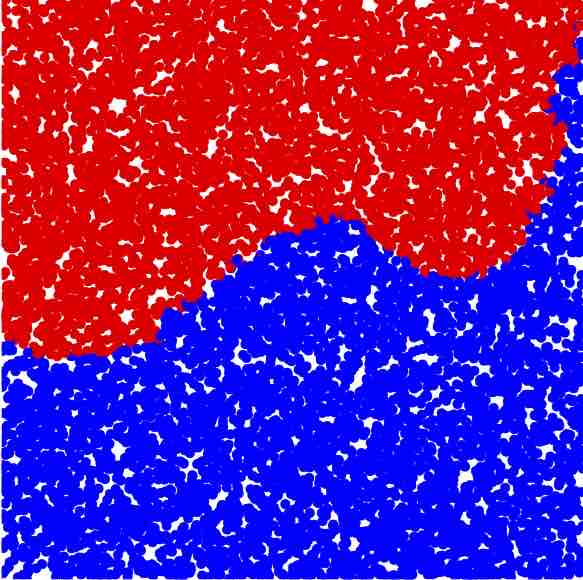}
}
\subcaptionbox{Smiley $k$-NN} {
\includegraphics[width=0.27\linewidth]{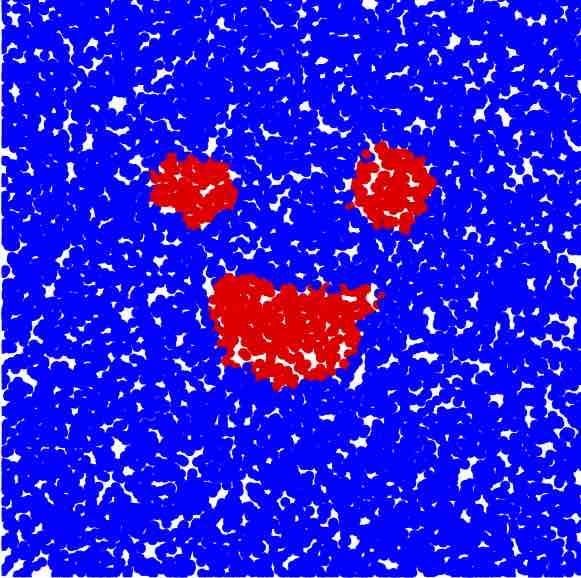}
}
\subcaptionbox{Bullseye $k$-NN} {
\includegraphics[width=0.27\linewidth]{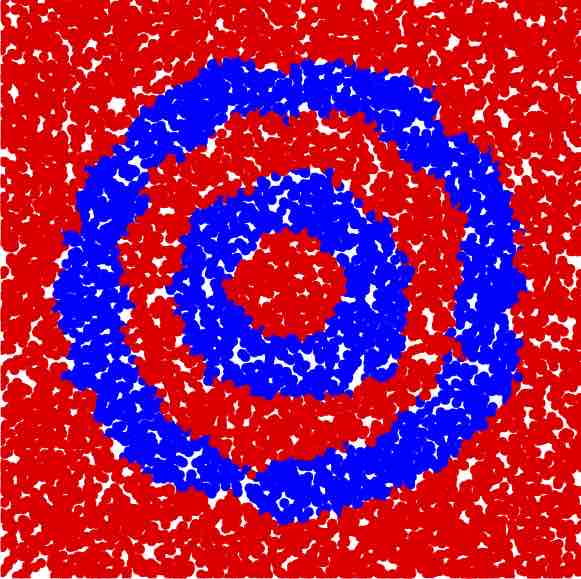}
}
\caption{\label{fig:TopArmResults}
Test results on top-arm identification using Ridge regression and 25-NN regression. Contexts are labeled in red and blue if arms 1 and 2 are estimated to be top respectively.
}
\end{figure}

\begin{table*}[th!]
\centering
\caption{\label{tab:hyperparam} Top-arm identification and regret results from Ridge and $k$-NN regressors. Each model was tuned individually and optimal hyperparameters are shown. $k$-NN performs better on both metrics for all three scenarios.}
\begin{tabular}{ccccccc} \toprule
&\multicolumn{2}{c}{\textbf{Quintic Function}}&\multicolumn{2}{c}{\textbf{Smiley}} &\multicolumn{2}{c}{\textbf{Bullseye}}\\
& Ridge&kNN &Ridge& kNN & Ridge&kNN \\
\midrule
\textbf{Top-Arm Test Error from Uniform Sampling}&
0.065& \textbf{0.002} & 0.080 & \textbf{0.000}& 0.335 & \textbf{0.005}\\
Number of samples& 500k & 500k  & 2k  & 5000k & 100k  & 500k   \\
Number of neighbors& -& 100 & - & 50 & - & 20 \\
\midrule\midrule
\textbf{Test Regret from UCB sampling}& 0.0315 & \textbf{0.001} & 0.0375 & \textbf{0.0135}& 0.161 & \textbf{0.004} \\
Number of samples& 1k & 500k  & 5k  & 1000k & 50k  & 1000k   \\
Number of neighbors& -& 100 & - & 20 & - & 100 \\
\bottomrule
\end{tabular}
\end{table*}

\subsubsection{Qualitative Analysis}
We first qualitatively show that $k$-NN regression can successfully model the bandits whereas the linear method cannot. The difficulty of the task is illustrated by Figure \ref{fig:RewardDensity}, which plots 10k uniformly sampled samples from each scenario with a colormap. We can see that a human would have a hard time recovering the regions where each arm is top due to the randomness in the observed rewards. This randomness is considerable as we set $\sigma=0.5$ to be the same as $|f_i(x\in R_i)-f_i(x\notin R_i)|$.

We fix the number of training samples $N$ to 10k and the number of nearest neighbors to $k=25$. We evaluate on 10k random test samples. Figure \ref{fig:TopArmResults} shows that $k$-NN regression does an excellent job of reproducing the region boundaries. Ridge regression does a poor job in the Quintic Function case, making a linear approximation to the quintic curve, and completely fails in the Smiley and Bullseye Cases, simply choosing the arm whose top-arm region is larger.

\subsubsection{Quantitative Analysis}
We report numerical results and optimal hyperparameters in Table \ref{tab:hyperparam}. We tuned other hyperparameters using grid search on a validation set of size 1k using grid search and we evaluate performance of our models on a test set of size 1k. We use the UCB strategy in \citet{auer2002nonstochastic} (a simplified version of UCB by \citet{agrawal2013thompson}). We found that a confidence level of $0.1$ worked well for all settings. We see that $k$-NN significantly outperforms Ridge regression for both top-arm identification and regret minimization in all three scenarios (Table \ref{tab:hyperparam}).


\subsection{Image Classification Experiments}
We extend our experiments to image classification of the canonical MNIST dataset, which consists of 60k training images and 10k test images of  isolated, normalized, hand-written digits. The task is
to classify each 28$\times$28 image into one of ten classes. We reframe this as a contextual MAB problem by treating the classes as arms and the images as the contexts. Note that for every context, the payoff of all arms are known: 1 if the class is the true label and 0 otherwise. We compare $k$-NN and Ridge regressions at regret minimization using the UCB strategy. As before we use the UCB strategy in \citet{auer2002nonstochastic} and fix the confidence level to 0.1. We do not employ any data augmentation.

We obtain test regret of 17.5\% from LinUCB with $\alpha=5$, where $\alpha$ is the coefficient of L2 regularization, and significantly lower test regret of 5.8\% from 4-NNUCB. Figure \ref{fig:mnist} shows that $k$-NN regression maintains lower regret than Ridge regression over a range of values of $k$ and $\alpha$. We note that Ridge regression working well for relatively large values of $\alpha$ itself suggests that it is a poor model for the task.

\begin{figure}[!htbp]
\centering
\includegraphics[width=0.93\linewidth]{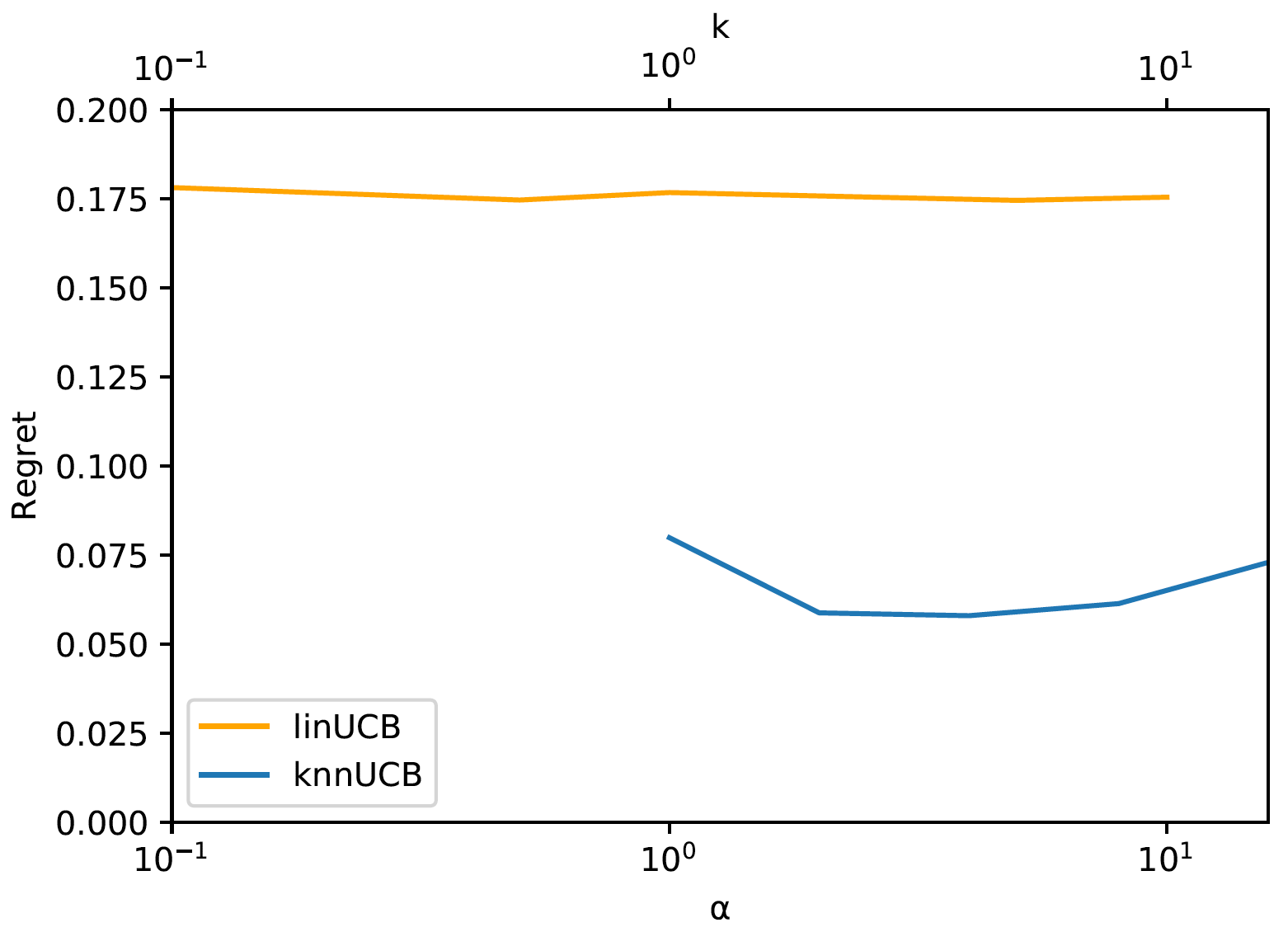}
\caption{\label{fig:mnist}
Test results on regret minimization for MNIST image classification over varied values of $\alpha$ (for LinUCB) and $k$ (for $k$-NNUCB). The nonparametric approach achieves significantly lower regret over a range of hyperparameters.}
\end{figure}

\section{Conclusion}
For the multi-armed bandit setting, we use nonparametric regression to attain tight results for top-arm identification and 
a sublinear regret of $\widetilde{O}(T^{\frac{1+D}{2+D}})$, where $D$ is the dimension of the context. We also show that if the underlying context space has a lower intrinsic dimension $d < D$, then our algorithm automatically adapts to the lower dimension and attains a faster rate of $\widetilde{O}(T^{\frac{1+d}{2+d}})$. We also provide a procedure for recovering the maximal connected regions in a support where a particular arm is the top-arm and provide a consistency analysis.
We then give a natural extension to infinite-armed contextual bandits. Our simulations confirm that our method is able to learn in the contextual setting with arbitrary decision boundaries, even in the presence of significant noise, and our experiments on classification of MNIST images demonstrate superior performance of our method over LinUCB on a real world task.  
\bibliographystyle{aaai}
\fontsize{9}{0}
\bibliography{bibfile}

\end{document}